\documentclass[10pt]{article} % For LaTeX2e

\usepackage[preprint]{rlj} % Should be uncommented for preprint versions

\usepackage{amssymb}            % Defines common symbols like \mathbb R
\usepackage{mathtools}          % Extends amsmath, providing common math tools
\usepackage{mathrsfs}           % Enables \mathscr, which can work in cases that \mathcal does not
\usepackage{graphicx}           % For including images
\usepackage{subcaption}         % Allows for the use of subfigures and subcaptions
\usepackage[space]{grffile}     % For spaces in image names
\usepackage{url}                % For displaying URLs
\usepackage{lipsum}             % For placeholder text

\usepackage{enumitem}
\usepackage{algorithm}
\usepackage{algorithmic}
\usepackage{multirow}
\usepackage{makecell}
\usepackage{tikz}
\usepackage{subcaption}
\usepackage{amsthm}
\usepackage{booktabs}

\usepackage{thmtools}
\usepackage{thm-restate}
\usepackage{amsmath,amsfonts,bm}

\def\eqref#1{equation~\ref{#1}}
\def\1{\bm{1}}

\DeclareMathAlphabet{\mathsfit}{\encodingdefault}{\sfdefault}{m}{sl}
\SetMathAlphabet{\mathsfit}{bold}{\encodingdefault}{\sfdefault}{bx}{n}

\def\gD{{\mathcal{D}}}

\def\gL{{\mathcal{L}}}

\def\gS{{\mathcal{S}}}

\def\sN{{\mathbb{N}}}

\def\sR{{\mathbb{R}}}

\def\sZ{{\mathbb{Z}}}

\newcommand{\E}{\mathbb{E}}

\newcommand{\Var}{\mathrm{Var}}

\newcommand{\Cov}{\mathrm{Cov}}
\DeclareMathOperator*{\argmin}{arg\,min}

\def\n1{\mathbf{n^{-1}}}
\def\rmat{\mathbf{R}}
\def\pmat{\mathbf{P}}
\def\mmat{\mathbf{M}}
\def\vvec{\mathbf{V}}
\def\avec{\mathbf{A}}

\newcommand{\statespace}{\mathcal{S}}
\newcommand{\actionspace}{\mathcal{A}}
\newcommand{\sufmca}{\mathrm{MC}\text{-}\mathrm{A}}

\newtheorem{theorem}{Theorem}
\newtheorem{lemma}{Lemma}
\newtheorem{corollary}{Corollary}
\title{On the Variance of Temporal Difference Learning and its Reduction Using Control Variates}

\setrunningtitle{On the Variance of TD Learning and its Reduction Using Control Variates}

\author{Hsiao-Ru Pan\textsuperscript{1}, Bernhard Sch\"olkopf\textsuperscript{1,2,3}}

\emails{\{hpan,bs\}@tuebingen.mpg.de}

\affiliations{
$^{1}$\textbf{Max Planck Institute for Intelligent Systems, T\"ubingen}\\
$^{2}$\textbf{ELLIS Institute T\"ubingen}\\
$^{3}$\textbf{ETH Z\"urich}\\
}

\contribution{
    We prove that, in the phased (synchronous) and the IID (asynchronous) settings, the asymptotic variance of temporal difference (TD) learning is upper-bounded by that of Monte Carlo methods, and that one mechanism underlying TD’s variance reduction is its effective aggregation over a larger pool of trajectories.
    }
    {
    TD learning is widely assumed to reduce variance via bootstrapping~\citep{sutton1998introduction, szepesvari2010algorithms, dann2014policy}.
    While previous works have shown that bootstrapping can be seen as a form of bias-variance tradeoff for multi-step TD~\citep{kearns2000bias}, and that batch TD is more efficient than MC~\citep{grunewalder2007optimality, cheikhi2023statistical}, the variance comparison is less clear in the online regime. We address this gap by analyzing TD’s variance reduction in the phased (synchronous) setting.
    }

\contribution{
    We show that (1) the advantage function can be used as control variates for \emph{policy evaluation}, and (2) Direct Advantage Estimation (DAE)~\citep{pan2022direct} is a form of control variate regression that simultaneously estimates both the value and advantage functions.
    }
    {
    While \citet{pan2022direct} demonstrated empirically that DAE improves deep RL performance, its theoretical justification has remained limited to convergence of the estimator. 
    We provide a complementary perspective by showing that DAE also reduces variance by analyzing it through the lens of control variates in policy evaluation, thereby offering a theoretical explanation for its empirical performance gains.
    }

 % Your keywords

\summary{
We analyze the variance of temporal difference (TD) learning using the phased setting with tabular representation, and show that one of the mechanisms behind its ability to reduce variance is by effectively aggregating over a larger number of independent trajectories.
Based on this insight, we demonstrate that (1) the variance of TD is asymptotically bounded from above by Monte Carlo (MC) estimators, and (2) shorter horizon updates incurs less variance for a fixed number of samples.
Beyond TD, we show that Direct Advantage Estimation (DAE), a method for estimating the advantage function, can be seen as a type of regression-adjusted control variate, which achieves a tighter bound on the variance compared to TD in the large-sample limit.
Finally, we numerically illustrate the behaviors of these estimators with carefully designed environments.
}

\begin{document}

\makeCover  % Create the cover page
\maketitle  % Make the title section

\begin{abstract}
We analyze the variance of temporal difference (TD) learning using the phased setting with tabular representation, and show that one of the mechanisms behind its ability to reduce variance is by effectively aggregating over a larger number of independent trajectories.
Based on this insight, we demonstrate that (1) the variance of TD is asymptotically bounded from above by Monte Carlo (MC) estimators, and (2) shorter horizon updates incurs less variance for a fixed number of samples.
Beyond TD, we show that Direct Advantage Estimation (DAE), a method for estimating the advantage function, can be seen as a type of regression-adjusted control variate, which achieves a tighter bound on the variance compared to TD in the large-sample limit.
Finally, we numerically illustrate the behaviors of these estimators with carefully designed environments.
\end{abstract}

%%%%%%%%%%%%%%%%%%%%%%%%%%%%%%%%%%%%%%%%%%%%%%%%%%%%%%%%%%%%%%%%
%% Section: Submission of papers to RLJ/RLC
%%%%%%%%%%%%%%%%%%%%%%%%%%%%%%%%%%%%%%%%%%%%%%%%%%%%%%%%%%%%%%%%
\section{Introduction}
\label{sec:intro}
Policy evaluation is a central problem in reinforcement learning (RL)~\citep{sutton1998introduction}.
Among various methods, temporal difference (TD) learning~\citep{sutton1988learning} stands out as a cornerstone technique for this class of problems.
Traditional methods like Monte Carlo (MC) methods estimate the expected return by averaging returns of sample trajectories, typically resulting in unbiased but high variance estimates.
In contrast, TD learning updates value estimates iteratively through bootstrapping (i.e., estimate based on previous estimates), thereby avoiding the need for full trajectories and tends to exhibit lower variance.
\citet{kearns2000bias} showed that bootstrapping can be seen as a form of bias-variance tradeoff, where high variance estimates from sample trajectories are replaced with low variance biased bootstrapped values.
While this intuition largely holds true, it is not difficult to see that the full story is more complicated.
For example, if we initialize the value estimates with the \emph{ground truth} values and update them using TD, then we essentially \emph{inject} variance (e.g., stochastic rewards) into the estimates, and it is not immediately clear how this variance affects bootstrapping asymptotically.
This shows that the way bootstrapping reduces variance may be more nuanced than simply "replace a high variance trajectory with a biased estimate".

In the present work, we analyze the variance of multi-step TD in the phased (synchronous) setting~\citep{kearns1998finite}, which abstracts away some of the complexities due to stochastic approximations.
Beyond TD, we also draw connections to Direct Advantage Estimation (DAE)~\citep{pan2022direct}, a method that simultaneously estimates the value function and the advantage function, and show that DAE can be seen as a type of control variate regression that enjoys a tighter bound on the variance compared to TD.
To summarize, we show that:
\begin{itemize}
    \item The asymptotic variance of multi-step TD is bounded above by that of MC, and one way TD reduces variance is by effectively averaging over a larger pool of trajectories.
    \item The advantage function can be seen as control variates for the MC estimator, which reduces its variance, and DAE is a type of regression-adjusted control variate for TD. 
\end{itemize}
Finally, we construct examples to illustrate the asymptotic behaviors of these estimators.

%%%%%%%%%%%%%%%%%%%%%%%%%%%%%%%%%%%%%%%%%%%%%%%%%%%%%%%%%%%%%%%%
%% Section: Citations, figures, tables, references, equations
%%%%%%%%%%%%%%%%%%%%%%%%%%%%%%%%%%%%%%%%%%%%%%%%%%%%%%%%%%%%%%%%
\section{Background}\label{sec:background}
We consider a discounted Markov Decision Process (MDP)~\citep{puterman2014markov} $(\statespace, \actionspace, p, r, \gamma)$ with finite state space $\statespace$, finite action space $\actionspace$, transition probability $p(s'|s, a)$, reward function $r(s,a)$, and discount factor $\gamma \in [0, 1)$.
For simplicity, we assume the reward function is deterministic unless otherwise stated and denote $r(s_t, a_t)$ by $r_t$ when the context is clear.
A policy $\pi(\cdot|s)$ is a function that maps states to distributions over $\actionspace$, and we focus on the case where the policy $\pi$ is fixed.
The value functions are defined by $V^\pi(s)=\E_\pi[\sum_{t=0}^\infty \gamma^t r_t|s_0{=}s]$, and $Q^\pi(s,a)=\E_\pi[\sum_{t=0}^\infty \gamma^t r_t|s_0{=}s, a_0{=}a]$, respectively ($\E_\pi$ indicates that actions are sampled from $\pi$).
The advantage function is defined by $A^\pi(s,a)=Q^\pi(s,a) - V^\pi(s)$.
For the present work, we focus on tabular representations, and use $\mathbf{V^\pi}\in\sR^{|\statespace|}$, $\mathbf{A^\pi}\in\sR^{|\statespace|\times|\actionspace|}$ to denote the value and the advantage functions, respectively, and use their functional forms for indexing.

Learning value functions is at the core of various policy optimization algorithms,
and one classical method is TD learning~\citep{sutton1988learning}, which updates the estimates through bootstrapping.
Specifically, one-step TD estimates the value function by sampling transition tuples $(s_t,a_t,s_{t+1})$ and updating the values by $V(s_t)\leftarrow V(s_t) + \alpha (r_t + \gamma V(s_{t+1}) - V(s_t))$, where $\alpha\in\sR$ is the learning rate.
An extension of this is multi-step TD, which samples multiple timesteps before updating the value via $V(s_t)\leftarrow V(s_t) + \alpha (r_t + \cdots + \gamma^k V(s_{t+k}) - V(s_t))$.
We denote this update by TD($k$)\footnote{This is not to be confused with TD($\lambda$).}.
%In the limit ($k\rightarrow\infty$), the method reduces to incremental Monte-Carlo (MC) methods.

\paragraph{Phased TD}
We presently consider the \emph{phased} setting~\citep{kearns1998finite}, which removes some of the complexities of TD leaning due to asynchronous updates and stochastic approximations.
In phased TD, value estimates are updated in \emph{phases}, where each phase consists of two steps:
(1) sampling $n$ $k$-step trajectories for each state: $\gD=\{\tau_{s,i}\}_{s\in\gS, i\in[1,\cdots,n]}$, where $\tau_{s,i} = (s, a_0^i, r_0^i, s_1^i, \cdots, s_k^i)$, %($n$ is assumed to be fixed throughout the paper unless otherwise stated)
and (2) updating the value of every state synchronously by
\begin{equation}\label{eqn:phased_td}
    V^{T+1}(s) \leftarrow \frac{1}{n} \sum_{i=1}^n \left(r_0^i + \cdots +\gamma^{k-1}r_{k-1}^i + \gamma^k V^T\left(s_k^i\right)\right),
\end{equation}
where $T$ denotes the phase.
It was shown that phased TD is analogous to \emph{TD learning with a fixed learning rate} under mild assumptions.
\citet{kearns2000bias} used this setting to analyze the bias-variance tradeoff of multi-step TD, and showed that the estimation error satisfies:
\begin{align*}
    \max_s|V^{T+1}(s) - V^\pi(s)|\coloneq\Delta_{T+1}\leq \underbrace{\left|\frac{1}{n} \sum_{i=1}^n \left(\sum_{t=0}^{k-1}\gamma^t r_t^i\right)  - \E\left[\left(\sum_{t=0}^{k-1}\gamma^t r_t\right)|s_0{=}s\right]\right|}_{\text{variance}} + \gamma^k \underbrace{\vphantom{\Bigg|}\Delta_T}_{\text{bias}},
\end{align*}
where the variance term accounts for the stochasticity from the sample rewards, while the bias term accounts for the error from bootstrapping.
%which accounts for the variance from sampled rewards, and a bias term $\gamma^k \Delta_{T-1}$, which accounts for the error from bootstrapping.
The authors then derived the following probabilistic bound (assume $\Delta_0 = 1$ and $|r(\cdot)| \leq 1$):
\begin{equation}\label{eqn:old_bound}
    \Delta_T\leq \frac{1-\gamma^{kT}}{1-\gamma}\sqrt{\frac{3\log(k/\delta)}{n}} + \gamma^{kT},
\end{equation}
with probability $1-\delta$.
Intuitively, increasing $k$ reduces the bias from bootstrapping at the cost of increased variances from the rewards.
However, this bias-variance decomposition fails to consider the variance from finite samples of $s_k^i$, and the variance from previous estimates $V^T$.
Furthermore, as $k\rightarrow\infty$ (MC estimation), this bound becomes vacuous even though the rewards are bounded.

\paragraph{Control Variate}
Control variate~\citep{Hammersley1964, asmussen2007stochastic} is a technique for reducing variances of MC methods.
We briefly review the basics of control variate in the one-dimensional setting, and refer the reader to~\citet{mcbook} for a more general treatment.

Recall that MC methods estimate $\mu_X=\E[X]$ by $\hat{\mu}_X = \frac{1}{n}\sum_{i=1}^n X_i$, where $X_i\overset{\mathrm{iid}}{\sim}X$.
Given another random variable $Y$, the control variate, that satisfies $\Cov(X,Y)\neq 0$ and has known $\E[Y]$ (assume $\E[Y]=0$, otherwise use $Z=Y-\E[Y]$).
Let $\hat{\mu}_{X,\lambda Y} = \frac{1}{n}\sum_{i=1}^n X_i + \lambda Y_i$, where $(X_i, Y_i)\overset{\mathrm{iid}}{\sim} (X,Y)$ and $\lambda$ is a tunable constant.
This new estimator remains unbiased, but has a different variance $\Var(\hat{\mu}_{X,\lambda Y}) = \Var(\hat{\mu}_X) + \frac{\lambda^2}{n}\Var(Y) + \frac{2\lambda}{n}\Cov(X, Y)$.
If we choose $\lambda = \lambda^* = -\frac{\Cov(X,Y)}{\Var(Y)}$, then 
$\Var(\hat{\mu}_{X,\lambda^* Y}) = \Var(\hat{\mu}_X) - \frac{\Cov(X,Y)^2}{n\Var(Y)} \leq \Var(\hat{\mu}_X)$,
which is never worse than the original estimator.
In practice, however, $\lambda^*$ is usually unknown, and estimated by $\hat{\lambda} = -\frac{\widehat{\Cov}(X,Y)}{\widehat{\Var}(Y)}$.
Interestingly, this estimator arises naturally from solving the following least squares:
\begin{equation}\label{eqn:racr}
    (\hat{\mu}_{X, \hat{\lambda}Y}, \hat{\lambda})=\argmin_{\theta,\lambda}\sum_{i=1}^n (\theta - x_i - \lambda y_i)^2.
\end{equation}
This approach is also known as regression-adjusted control variate, and can be readily generalized to more complex settings (e.g., multiple control variates).
The error of this estimator is:
\begin{align}
    \hat{\mu}_{X,\hat{\lambda} Y} - \mu_X  = \underbrace{(\hat{\lambda} - \lambda^*)\sum_{i=1}^n \frac{Y_i}{n}}_{O(1/n)} + \underbrace{\vphantom{\sum_{i=1}^n \frac{Y_i}{n}}\hat{\mu}_{X,\lambda^* Y} - \mu_X}_{O(1/\sqrt{n})}.
\end{align}
As both $\hat{\lambda} - \lambda^*$ and $\sum_{i=1}^n \frac{Y_i}{n}$ approach zero with rates $O(1/\sqrt{n})$, their product converges to 0 with rate $O(1/n)$.
Consequently, the second term dominates the error asymptotically, and $\hat{\mu}_{X,\lambda^* Y}$ can be seen as a large-sample approximation of $\hat{\mu}_{X,\hat{\lambda} Y}$.
Finally, we note that the estimator is, in general, biased since $Y_i$ can be correlated with $\hat{\lambda}$ (i.e., $\E[Y_i \hat{\lambda}]\neq 0$); however, the estimator remains consistent as the errors approach zero for $n\rightarrow\infty$.

\paragraph{Direct Advantage Estimation}
Direct Advantage Estimation (DAE)~\citep{pan2022direct} is a method for estimating the advantage function directly from sample trajectories.
Similar to multi-step TD, DAE can update values by bootstrapping previous estimates.
Specifically, DAE estimates the values by iteratively minimizing the following constrained least-squares:
\begin{gather}\label{eqn:loss_bs}
    \gL(\hat{A}, \hat{V})=\E_\pi
    \left[
        \left(
            \sum_{t=0}^{k-1}\gamma^{t}(r_{t} - \hat{A}_{t}) + \gamma^{k}V_\mathrm{tar}(s_{k})-\hat{V}(s_0)
        \right)^2
    \right]    
    \; \mathrm{s.t.}
    \sum_{a}\pi(a|s)\hat{A}(s,a)=0,
\end{gather}
where $\hat{A}_{t}=\hat{A}(s_t, a_t)$, and $V_\mathrm{tar}$ is the bootstrapping target.
It was shown that the minimizer of this objective can be viewed as multi-step estimates of the advantage function and the value function.
In practice, the expectation is replaced with an average over sample trajectories.
We note that, if we force $\hat{A}\equiv 0$ and only optimize with respect to $\hat{V}$, then DAE reduces to multi-step TD.
DAE demonstrated strong empirical performance in the deep RL setting; however, it remains unclear whether estimating the value function this way is beneficial compared to classical approaches.

\begin{table*}[!t]
    \caption{Variables names and their definitions. $i$ denotes the $i$th trajectory. $k$ and $n$ are constants unless otherwise stated. Note that $\pmat$ is a (row) vector of dimension $|\statespace|$.}
    \centering
    \begin{tabular}{ccc}
        \toprule
        Variable & Def. & Description \\ \hline
        $\mathbf{G}$ & $\frac{1}{n}\sum_{i=1}^n \sum_{t=0}^\infty \gamma ^t r_t^i$ &  Average of sample returns \\
        $\gD$ & $\left\{(s, a_0^i, r_0^i, s_1^i, \cdots, s_k^i)\right\}_{s\in\statespace, i=1,...,n}$ & $k$-step partial trajectories \\
        $\rmat$ & $\frac{1}{n}\sum_{i=1}^n\sum_{t=0}^{k-1} \gamma ^t r_t^i$ & Average of sample $k$-step rewards \\
        $\pmat$ & $\pmat_{s'}=\frac{1}{n}\sum_{i=1}^n\mathbb{I}(s_k^i=s')$  & Empirical $k$-step transition distribution \\
        \bottomrule
    \end{tabular}
    \label{tab:notations}
\end{table*}

\section{Variance of Monte Carlo (MC) and TD Learning}\label{sec:var_mctd}
We begin by showing that the variance of MC can be broken down into segments.
This will become useful when compared to multi-step TD, whose variance follows a similar form.

For clarity, we summarize the notations used in subsequent analyses in Table~\ref{tab:notations}.

\paragraph{Variance of MC}
Recall that MC estimates values via $V_{\mathrm{MC}}(s) = \mathbf{G}$.
The following lemma shows that its variance is equal to the sum of the variances of partial trajectories.
\begin{lemma}
$\Var(V_\mathrm{MC}(s))= \Var(\mathbf{G}) =
        \sum_{m=0}^\infty \gamma^{2km}\E\left[\left.\Var(\rmat + \gamma^k \pmat \vvec^\pi|s_0{=}s_{km})\right|s_0{=}s\right]$,
\label{lemma:var_mc}
\end{lemma}
This is a multi-step ($k\geq 1$) multi-sample ($n\geq 1$) generalization of the variance recursion derived by~\citet{sobel1982variance}, see Appendix~\ref{app:proof_lemma_1} for a proof.
As we will see, the variance of TD shares a similar structure, except $\vvec^\pi$ is replaced by bootstrapping targets.

\paragraph{Variance of Multi-Step TD}
We first rewrite the phased update (Equation~\ref{eqn:phased_td}) into:
\begin{align}
    \label{eqn:phased_td_update}
    V^T_{\mathrm{TD}(k)}(s) = \rmat + \gamma^k \pmat\vvec^{T-1}_{\mathrm{TD}(k)}.
\end{align}
Without loss of generality, we set $\vvec^0_{\mathrm{TD}(k)}\equiv 0$.
Note that $\vvec^T_{\mathrm{TD}(k)}$ is now a random vector for $T>0$ as the update involves random variables $\rmat$ and $\pmat$, and its expectation, denoted $\bar{\vvec}^T_{\mathrm{TD}(k)}$, is equal to:
\begin{align*}
    \bar{V}^T_{\mathrm{TD}(k)}(s) 
    = \E\left[\rmat\right] + \gamma^k \E\left[\pmat\right]\E\left[\vvec^{T-1}_{\mathrm{TD}(k)}\right]
    = \E\left[\left.\sum_{t=0}^{k-1}\gamma^t r_t + \gamma^k \bar{V}^{T-1}_{\mathrm{TD}(k)}(s_k)\right|s_0{=}s\right],
\end{align*}
and we recover the $k$-step Bellman update.
We emphasize that, while this expectation converges to $\vvec^\pi$ in the limit ($T\rightarrow\infty$), $\vvec^T_{\mathrm{TD}(k)}$ in general does \emph{not}.
This is due to the finite-sample variance, which causes the estimates to oscillate around $\vvec^\pi$.
We now analyze this variance and show that, similar to Lemma~\ref{lemma:var_mc}, the upper bound of the variance of TD also satisfies a recurrence relation.
\begin{lemma}\label{lemma:var_td}
    \begin{equation}
    \Var(V^T_{\mathrm{TD}(k)}(s)) \leq  
    \Var\left(\left. \rmat + \gamma^k\pmat\bar{\vvec}^{T-1}_{\mathrm{TD}(k)}\right|s_0{=}s\right) + \gamma^{2k}\E_{s_k}\left[\left.\Var\left(V^{T-1}_{\mathrm{TD}(k)}(s_k)\right)\right|s_0{=}s\right].    
    \end{equation}
\end{lemma}
\begin{proof}
\begin{align}
    &\Var(V^T_{\mathrm{TD}(k)}(s)) = \Var\left(\left. \rmat + \gamma^k \pmat\vvec^{T-1}_{\mathrm{TD}(k)}\right|s_0{=}s\right)\\
    = &\Var\left(\left.\E\left[\left. \rmat + \gamma^k \pmat\vvec^{T-1}_{\mathrm{TD}(k)}\right|\gD\right]\right|s_0{=}s\right) 
     +\E\left[\left.\Var\left(\left.\rmat + \gamma^k \pmat\vvec^{T-1}_{\mathrm{TD}(k)}\right|\gD\right)\right|s_0{=}s\right] \\
    =& \Var\left(\left. \rmat + \gamma^k\pmat\bar{\vvec}^{T-1}_{\mathrm{TD}(k)}\right|s_0{=}s\right)
     +\gamma^{2k}\E\left[\left.\Var\left(\left.\pmat\vvec^{T-1}_{\mathrm{TD}(k)}\right|\gD\right)\right|s_0{=}s\right] \\
    \leq& \Var\left(\left. \rmat + \gamma^k\pmat\bar{\vvec}^{T-1}_{\mathrm{TD}(k)}\right|s_0{=}s\right)
     +\gamma^{2k}\E\left[\left.\pmat\Var\left(\vvec^{T-1}_{\mathrm{TD}(k)}\right)\right|s_0{=}s\right]\\
    =&\Var\left(\left. \rmat + \gamma^k\pmat\bar{\vvec}^{T-1}_{\mathrm{TD}(k)}\right|s_0{=}s\right) 
     + \gamma^{2k}\E_{s_k}\left[\left.\Var\left(V^{T-1}_{\mathrm{TD}(k)}(s_k)\right)\right|s_0{=}s\right].
\end{align}    
\end{proof}
If we expand this recursion, then
\begin{align}
    \Var(V^T_{\mathrm{TD}(k)}(s)) \leq 
    \sum_{m=0}^{T-1}\gamma^{2km} \E\left[\left.\Var\left(\left.\rmat + \gamma^k\pmat\bar{\vvec}^{T-1-m}_{\mathrm{TD}(k)}\right|s_0{=}s_{km}\right)\right|s_0{=}s\right],
\end{align}
and we arrive at a similar expression to the MC estimator (Lemma~\ref{lemma:var_mc}). In fact, we get, in the limit:
\begin{theorem}
    \label{thm:td_mc}
    $\lim_{T\rightarrow\infty}\Var(V^T_{\mathrm{TD}(k)}(s))\leq \Var(V_\mathrm{MC}(s))$.
\end{theorem}

\begin{figure}[t]
    \centering
    \begin{subfigure}[t]{0.48\textwidth}
        \centering
        \begin{tikzpicture}[scale=1]
        
        \node[circle, minimum size=7mm] (s0) at (0, 0) {$s_0$};
        \node[circle, minimum size=7mm] (sk1) at (-.7, .6) {$s_k^1$};
        \node[circle, minimum size=7mm] (sk2) at (1.5, .6) {$s_k^2$};
        \node[circle, minimum size=7mm] (sk3) at (1.5, -.3) {$s_k^3$};
        
        \draw[->, out=30, in=0] (s0) to (sk1);
        \draw[->, out=180, in=0, dashed] (sk1) to (-2.5, 0.7);
        \draw[->, out=180, in=0, dashed] (sk1) to (-1.5, 1.8);
        \draw[->, out=180, in=0, dashed] (sk1) to (-2, 1.7);
        \draw[->, out=30, in=190] (s0) to (sk2);
        \draw[->, out=100, in=-90, dashed] (sk2) to (.6, 1.6);
        \draw[->, out=100, in=-90, dashed] (sk2) to (1.3, 2);
        \draw[->, out=100, in=-90, dashed] (sk2) to (2.3, 1.7);
        \draw[->, out=30, in=160] (s0) to (sk3);
        \draw[->, out=0, in=160, dashed] (sk3) to (3, .5);
        \draw[->, out=0, in=140, dashed] (sk3) to (3, 0);
        \draw[->, out=0, in=170, dashed] (sk3) to (3, -.5);
        
        \end{tikzpicture}    
        \subcaption{Independent bootstrapping states}
    \end{subfigure}    
    \begin{subfigure}[t]{0.48\textwidth}
        \centering    
        \begin{tikzpicture}[scale=1]
        
        % \def\n{5}
    
        % \foreach \i in {1, ..., \n}{
        %     \node [draw, circle] (N\i) at (\i * -360 / \n + 360 / \n:2) {\i};
        % }
        % \pgfmathtruncatemacro\m{\n-1};
        % \foreach \i in {1, ..., \m }{
        %     \pgfmathtruncatemacro\k{\i+1};
        %     \draw (N\i) edge [->, bend left] node [font=\small, fill=white] {-1} (N\k);
        %     \draw (N\i) edge [->, bend right] node [font=\small, fill=white] {1} (N\k);    
        % }   
    
        \node[circle, minimum size=7mm] (s0) at (0, 0) {$s_0$};
        \node[circle, minimum size=7mm] (ski) at (1.5, .8) {$s_k^i$};
        
        \draw[->, out=10, in=-50] (s0) to (ski);
        \draw[->, out=10, in=230] (s0) to (ski);
        \draw[->, out=10, in=190] (s0) to (ski);
        \draw[->, out=30, in=190, dashed] (ski) to (2.7, 1.6);
        \draw[->, out=30, in=200, dashed] (ski) to (2.8, 1.3);
        \draw[->, out=30, in=170, dashed] (ski) to (2.9, 0.8);
        
        \end{tikzpicture}
        \subcaption{Correlated bootstrapping states}
    \end{subfigure}
    \caption{Illustration of how TD can reduce variance. Solid and dashed arrows represent trajectories collected in the current phase and the previous phase, respectively.
    (a) $s_k^i$ are independent, and averaging their values effectively aggregates over a larger pool of trajectories. 
    (b) $s_k^i$ are all the same, and averaging their values provides no variance reduction, as the effective number of trajectories remains the same.}
    \label{fig:td_correlation}
\end{figure}
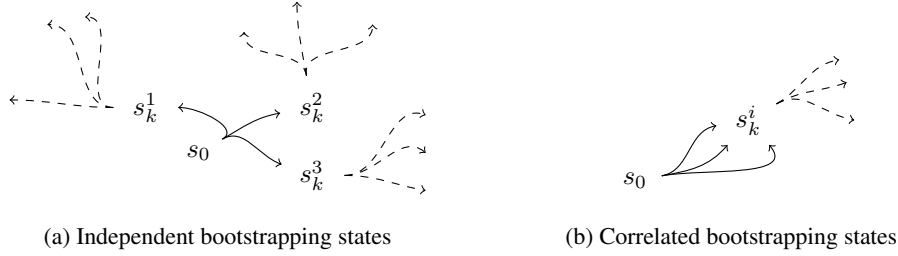

See Appendix~\ref{app:proof_thm1} for a proof and a discussion of how this can be extended to the asynchronous IID setting.
This shows that, asymptotically, TD is no worse than MC, independent of $k$; however, it also suggests that TD, in the worst case, can suffer from the same variance as MC.

\paragraph{When does the equality hold?}
Let us examine the only inequality used in the derivation, namely, $\Var\left(\left.\pmat\vvec^{T-1}_{\mathrm{TD}(k)}\right|\gD\right) \leq \pmat\Var\left(\vvec^{T-1}_{\mathrm{TD}(k)}\right)$.
Recall that the left-hand side is simply the variance of the average of the bootstrap values, namely $\Var\left(\left.\frac{1}{n}\sum_iV^{T-1}_{\mathrm{TD}(k)}(s_k^i)\right|\gD\right)$, and this inequality becomes an equality when the value estimates have correlation 1 (e.g., $s_k^i$ are the same for all $i$).
This suggests that one way bootstrapping reduces variance is by effectively aggregating over diverse states with weak dependencies between their value estimates, as illustrated in Figure~\ref{fig:td_correlation}.
In Section~\ref{sec:exp}, we also construct toy examples to illustrate this effect.

\paragraph{A remark on sample efficiency}
It is not immediately clear whether the comparison in Theorem~\ref{thm:td_mc} is fair in terms of sample efficiency as the result is asymptotic.
Here, we compare two settings with the same number of samples: (1) $T$ phases with $k$-step updates, and (2) $1$ phase with $kT$-step update.
Both settings require $nkT |\statespace|$ number of samples, and share the same expected value given the same initialization.
However, their variances satisfy the following ordering:
\begin{corollary}
    \label{cor:sample_efficiency}
    For $\tau \geq 0$, we have
    $\Var(V^{\tau+T}_{\mathrm{TD}(k)}(s)|\vvec^{\tau}_{\mathrm{TD}(k)}{=}\vvec) \leq 
        \Var(V^{\tau+1}_{\mathrm{TD}(kT)}(s)|\vvec^{\tau}_{\mathrm{TD}(kT)}{=}\vvec)$.
\end{corollary}
See Appendix~\ref{app:proof_cor1} for a proof.
This shows that short horizon updates incur less variance than a single long horizon update given the same number of interactions and bootstrapping targets.

\section{Control Variate and the Advantage Function}
We now shift our focus to the second main contribution of the present work, namely, how to use the advantage function to reduce the variance of value estimations.

We begin by considering the $\pi$-centered function class $F_\pi = \{f| \E_\pi\left[f(s,a)|s\right]=0\;\forall s\}$, which has the following property (given $f\in F_\pi$): $\E_\pi\left[\sum_{t=0}^\infty \gamma^t \left(r_t - f(s_t,a_t)\right)\right] = \E_\pi\left[\sum_{t=0}^\infty \gamma^t r_t\right]$.
In other words, introducing $f$ does not bias the MC estimate, and $f$ can be seen as a control variate.
A natural question is, then, what would be the optimal choice of $f^*$ that minimizes the variance, that is, $f^* = \argmin_{f\in F_\pi}\Var\left(\sum_{t=0}^\infty \gamma^t\left(r_t - f(s_t,a_t)\right)\right)$.
\citet{pan2022direct} proved that the advantage function $A^\pi$ is the unique minimizer under mild assumptions, and used it as a way to estimate $A^\pi$.
Now, if $A^\pi$ is known, then we can construct a new estimator via:
\begin{align}\label{eqn:dae_est}
    V_{\sufmca}(s) = \frac{1}{n} \sum_{i=1}^n \left(\sum_{t=0}^\infty \gamma ^t \left(r_t^i-A^\pi(s_t^i, a_t^i)\right)\right).
\end{align}
One may wonder to what extent can this control variate reduce the variance.
To answer this, we use the return decomposition by~\citet{pan2024skill}:
\begin{equation}
    \sum_{t=0}^\infty \gamma^t r_t = V^\pi(s_0) + \sum_{t=0}^\infty \gamma^t (A^\pi(s_t, a_t) + B^\pi(s_t, a_t, s_{t+1})),
\end{equation}
where $B^\pi(s_t, a_t, r_t,s_{t+1}) = r_t + \gamma V^\pi(s_{t+1}) - \E[r + \gamma V^\pi(s')|s_t, a_t]$ quantifies how much of the return is caused by stochastic transitions\footnote{We use a slightly more general definition to incorporate stochastic rewards}.
If $B^\pi\equiv 0$ (e.g., the environment is deterministic), then:
\begin{equation}\label{eqn:decomp_det}
    \sum_{t=0}^\infty \gamma^t r_t = V^\pi(s_0) + \sum_{t=0}^\infty \gamma^t A^\pi(s_t, a_t).
\end{equation}
Combine this with Equation~\ref{eqn:dae_est}, we have $\Var(V_{\sufmca}(s)) = 0$, meaning that $A^\pi$ can fully explain away the variance of MC in this case.
For more general cases, $B^\pi$ is required to account for the variance caused by stochastic transitions, and it remains open whether $B^\pi$ can be easily estimated in model-free settings.
As such, we focus on the advantage function in the present work.

In practice, $A^\pi$ is rarely known a priori, so we have to estimate both $V^\pi$ and $A^\pi$ simultaneously.
Next, we show that DAE~\citep{pan2022direct} can be seen as a type of regression-adjusted control variate, which achieves this.

\subsection{Direct Advantage Estimation and Control Variate Regression}\label{sec:dae_and_vr}
DAE estimates $V^\pi$ and $A^\pi$ by solving a constrained least-square problem (Equation~\ref{eqn:loss_bs}).
For the present work, we focus only on the value estimates and treat $\hat{A}$ as nuisance parameters.
This allows us to remove the constraint and reformulate the objective into:
\begin{gather*}\label{eqn:loss_bs_reform}
    \gL_T(\hat{A}, \hat{V}) = 
    \E_\pi
    \left[
        \left(
            \sum_{t=0}^{k-1}\gamma^{t}
            \left(
                r_{t} - \breve{A}_t
                %\left(
                %    \hat{A}_t - \sum_{a}\pi(a|s_t)\hat{A}(s_t, a)
                %\right)
            \right) +
            \gamma^{k}V^T_{\mathrm{DAE}(k)}(s_{k}) - 
            \hat{V}(s)
        \right)^2
    \right], 
\end{gather*}
where $\breve{A}_t = \hat{A}(s_t, a_t) - \sum_a \pi(a|s_t) \hat{A}(s_t, a)$ (this parametrization ensures that $\sum_a \pi(a|s) \breve{A}(s,a) = 0$).
Under this formulation, the minimizer of $\hat{A}$ may no longer be unique, but the minimizer of $\hat{V}$ remains unchanged.
Let us now consider DAE in the phased setting, where the objective becomes:
\begin{equation*}
    \sum_{i=1}^n
    \left(
        \sum_{t=0}^{k-1}\gamma^{t}
        \left(
            r_{t}^i - \breve{A}_t^i
            %\left(
            %    \hat{A}_t - \sum_{a}\pi(a|s_t)\hat{A}(s_t, a)
            %\right)
        \right) +
        \gamma^{k}V^T_{\mathrm{DAE}(k)}(s_{k}^i) - 
        \hat{V}(s)
    \right)^2.
\end{equation*}
Let $(\vvec^{T+1}_{\mathrm{DAE}(k)}, \avec^{T+1})$ be a minimizer of this objective, and $\mmat\in\sR^{|\statespace|\times|\actionspace|}$ be the following:
\begin{equation*}
    \mmat_{(s,a)} =\frac{1}{n}\sum_{i=1}^n \sum_{t=0}^{k-1} \gamma^t 
    \left(
        \mathbb{I}(s^i_t{=}s, a^i_t{=}a) - \pi(a|s)\mathbb{I}(s^i_t{=}s)
    \right),
\end{equation*}
which compares the empirical occupancy measure to its expectation over the actions, then
\begin{equation}
    \label{eqn:phased_dae_update}
    V^{T+1}_{\mathrm{DAE}(k)}(s) = \rmat + \gamma^k \pmat\vvec^{T}_{\mathrm{DAE}(k)} - \mmat\avec^{T+1}.
\end{equation}
Note that, this update rule differs from the TD update (cf. Equation~\ref{eqn:phased_td_update}) only by the term $\mmat\avec^{T+1}$.
Comparing this to Equation~\ref{eqn:racr}, we see that DAE is a case of regression-adjusted control variate, where $\mmat$ is the control variate with $\E[\mmat]=0$, and $\avec^T$ is the corresponding coefficients. 
Additionally, recall that this regression estimator ($\vvec^T_{\mathrm{DAE}(k)}$) behaves similarly to the one with optimal control variate coefficients ($\vvec^T_{\mathrm{DAE}^*(k)}$) under the large-sample approximation (cf. Section~\ref{sec:background}).
Here, we focus on $\vvec^T_{\mathrm{DAE}^*}$ and leave it for future work to analyze the small-sample bias due to control variates.\footnote{The large-sample approximation is common in the literature~\citep{lin2013agnostic, mcbook, davidson1992regression}, since the small-sample bias from imperfect control variate coefficients is often much smaller than the variance.}
The following result shows the asymptotic behavior of DAE:
\begin{theorem}
    \label{thm:dae}
    $\limsup_{T\rightarrow\infty} \Var(V^{T}_{\mathrm{DAE}^*(k)}(s)) 
    \leq \Var(V_{\sufmca}(s)) $
\end{theorem}
See Appendix~\ref{app:proof_thm2} for a proof.
Compared to Theorem~\ref{thm:td_mc}, this shows that DAE enjoys a tighter upper bound on the variance by using control variates (recall that $\Var(V_{\sufmca}(s))\leq\Var(V_{\mathrm{MC}}(s))$).
Finally, we note that since TD is a special case of Equation~\ref{eqn:phased_dae_update} with $\hat{A}\equiv 0$, it follows that
$\Var(V^{T+1}_{\mathrm{DAE}^*(k)}(s)|\vvec^{T}_{\mathrm{DAE}^*(k)}{=}\vvec)
    \leq \Var(V^{T+1}_{\mathrm{TD}(k)}(s)|\vvec^{T}_{\mathrm{TD}(k)}{=}\vvec)$.

\section{Empirical Illustration}\label{sec:exp}

\begin{figure}[t]
    \begin{minipage}{.49\linewidth}
        \centering
        \begin{tikzpicture}
        
        % \def\n{5}
    
        % \foreach \i in {1, ..., \n}{
        %     \node [draw, circle] (N\i) at (\i * -360 / \n + 360 / \n:2) {\i};
        % }
        % \pgfmathtruncatemacro\m{\n-1};
        % \foreach \i in {1, ..., \m }{
        %     \pgfmathtruncatemacro\k{\i+1};
        %     \draw (N\i) edge [->, bend left] node [font=\small, fill=white] {-1} (N\k);
        %     \draw (N\i) edge [->, bend right] node [font=\small, fill=white] {1} (N\k);    
        % }   
    
        \node[draw, circle, minimum size=8mm] (s1) at (0, 0) {1};
        \node[draw, circle, minimum size=8mm] (s2) at (1.5, 0) {2};
        \node[circle, minimum size=8mm] (sinv) at (3, 0) {};
        \node[circle, minimum size=8mm] at (.75, 0.1) {\vdots};
        \node[circle, minimum size=8mm] at (2.25, 0.1) {\vdots};
        \node[draw, circle, minimum size=8mm] (s8) at (4.25, 0) {8};
        \node[circle, minimum size=8mm] at (5, 0.1) {\vdots};
        \node (cdots) at (3.5, 0) {$\cdots$};
    
        \draw[->, bend right=60] (s1) to (s2);
        \draw[->, bend left=60] (s1) to (s2);
    
        \draw[->, bend right=60] (s2) to (sinv) ;
        \draw[->, bend left=60] (s2) to (sinv);
    
        \draw[->, bend right=60] (s8) to (5.75,0) ;
        \draw[->, bend left=60] (s8) to (5.75,0);
    
        \draw[->] (5.75, 0) -- (5.75, -1.3) -- (0, -1.3) -- (s1);
        
        \end{tikzpicture}
        \captionof{figure}{Chain MDP with $\statespace=\{1,2,..., 8\}$, $\actionspace=\{1,...,|\actionspace|\}$, and $r(s,a) = \frac{(-1)^a}{4}$ (independent of $s$). Agents return to state $1$ after state $8$.}
        \label{fig:chain_env}
    
    \end{minipage}
    \begin{minipage}{.49\linewidth}        
        \centering
        \captionof{table}{Parameters of the experiment.}
        \begin{tabular}{cc}
            \toprule
            Param. & Description \\ \hline
            $|\actionspace|$ & action space size \\
            $n$ & number of sample trajectories \\
            $k$ & backup length \\
            $p_r$ & probability of reward masking \\
            $p_s$ & probability of sticky transition \\
            \bottomrule
        \end{tabular}            
        \label{tab:pars}
    \end{minipage}
    
\end{figure}

In this section, we illustrate the behaviors of different estimators through experiments based on variants of the chain environment shown in Figure~\ref{fig:chain_env}.
Despite its simplicity, the environment is sufficiently expressive to illustrate various properties of the estimators analyzed in the present study, such as how the variances depend on the bootstrapping states.

All experiments are based on the phased setting, where values are updated synchronously at the end of each phase.
We fix the number of phases at 2500, which we found sufficient convergence, the policy $\pi$ to be uniform, the discount factor at $\gamma=0.99$, and consider two types of stochasticity:

\begin{enumerate}[label=\arabic*., itemsep=0pt, wide=0pt, left=0pt]
    \item \textbf{Reward masking}: Rewards are masked out with probability $p_r$ (i.e., $r=0$ with probability $p_r$).
    \item \textbf{Sticky transition}: With probability $p_s$, the agent stays in the current state instead of advancing to the next state (i.e., $s_{t+1}=s_t$ with probability $p_s$).
\end{enumerate}

For simplicity, we consider cases where $|\actionspace|$ is even, such that $V^\pi\equiv 0$.
The variances of the MC and the MC-A estimators are equal to
\begin{align}
    \Var(V_\mathrm{MC}(s)) = \frac{(1-p_r)}{16(1-\gamma^2)n},\;\text{and}\;    \Var(V_{\sufmca}(s)) = p_r\Var(V_\mathrm{MC}(s)),
\end{align}
respectively, and will be used as baselines for comparing the estimators.
We note that the variances do not depend on $p_s$ due to the symmetric nature of the states.

Each run (configuration) is repeated for 1000 different random seeds to ensure statistical significance.
We compare the mean squared error between the true value function $V^\pi$ and the estimated value ($\mathrm{MSE}(\hat{V}, V^\pi)$, $\mathrm{MSE}_\pi$ for short) , which reflects how quickly the estimator converges to $V^\pi$.
We also measure the error between the estimated value and the expected Bellman update at each phase ($\mathrm{MSE}(\hat{V}, V_\mathrm{BE})$, $\mathrm{MSE}_\mathrm{BE}$ for short).
For TD, $\mathrm{MSE}_\mathrm{BE}$ corresponds to the variance of the update, since the expected TD update follows the Bellman update (i.e., $\E[\vvec^T_\mathrm{TD}] = \vvec^T_\mathrm{BE}$).
For DAE, it additionally captures the small-sample bias introduced by the regression estimator.

% exp 1 deterministic -> td approach MC, dae approach 0
\begin{figure*}[t]
    \centering
    \includegraphics[width=0.98\linewidth]{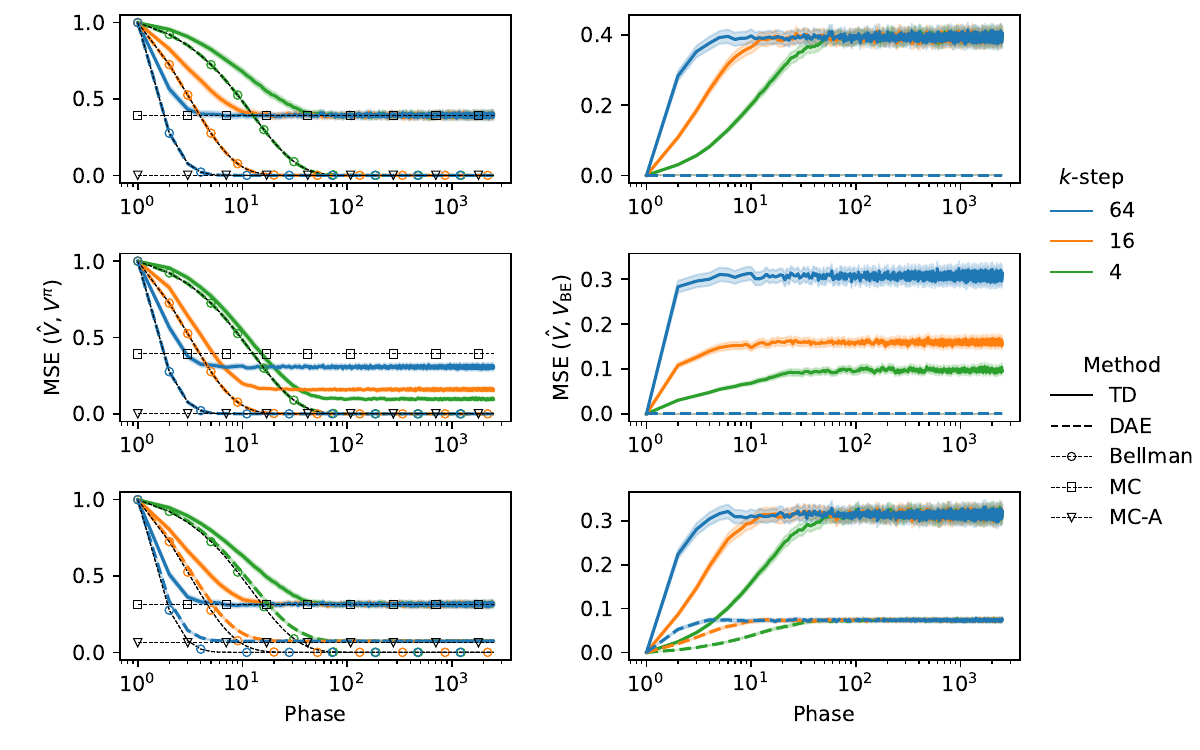}
    \caption{The \textbf{Deterministic} case (top), the \textbf{Sticky Transition} case (middle), and the \textbf{Reward Masking} case (bottom). Lines and shadings represent (mean $\pm$ 3 standard error).}%
    \label{fig:deterministic_and_stoch}
\end{figure*}

\paragraph{The Deterministic Case \normalfont{($|\actionspace|=2$, $n=8$, $k\in\{4, 16, 64\}$, $p_s=0$, $p_r=0$)}}

We first examine the effect of $k$ in the simplest setting with a deterministic environment.
Since the transitions are fully deterministic and independent of actions, the bootstrapping states are identical for a given starting state.
Our analysis (Section~\ref{sec:var_mctd}) suggests that bootstrapping loses its ability to reduce variance when the value estimates of the bootstrapping states are fully correlated, and TD would behave like MC asymptotically in this case.
Indeed, Figure~\ref{fig:deterministic_and_stoch} (top) shows that TD learning, independent of $k$, approaches the same $\mathrm{MSE}_\pi$ as MC.
Note that this cannot be explained by the bound given by Equation~\ref{eqn:old_bound}, which predicts that the asymptotic error would grow as $k$ increases.
From $\mathrm{MSE}_\mathrm{BE}$, we also see the effect of variance injection, where the variances increase with respect to the number of updates.
On the other hand, we see DAE achieving almost zero $\mathrm{MSE}_\mathrm{BE}$, demonstrating the effectiveness of the variance reduction.

\paragraph{The Sticky Transition Case \normalfont{($|\actionspace|=2$, $n=8$, $k\in\{4, 16, 64\}$, $p_s=0.25$, $p_r=0$)}}
We now examine how stochastic transitions affect the estimators.
Figure~\ref{fig:deterministic_and_stoch} (middle) shows that, counterintuitively, stochasticity \emph{reduces} the variance of TD.
Furthermore, the learning curves now follow the common bias-variance tradeoff intuition of multi-step learning (i.e., larger $k$ learns faster but leads to higher variance and vice versa).
We emphasize that this setting effectively differs from the deterministic case only in how the bootstrapping states are sampled.
As such, the variance reduction can only be explained by TD learning's ability to effectively aggregate value estimates over more diverse states (cf. Figure~\ref{fig:td_correlation}).
DAE converges with rates similar to Bellman iterations again, since $B^\pi \equiv 0$ does not depend on $p_s$.

So far, both cases have almost full coverage of the state-action space from the data in each phase, and have variances that can be fully explained by the advantage function (cf. Equation~\ref{eqn:decomp_det}).
This allowed DAE to converge almost as fast as Bellman iterations.
To see when this breaks down, we next consider the stochastic reward ($B^\pi\not\equiv 0$) and the large $|\actionspace|$ settings.

\paragraph{The Reward Masking Case \normalfont{($|\actionspace|=2$, $n=8$, $k\in\{4, 16, 64\}$, $p_s=0$, $p_r=0.2$)}}
Similar to the deterministic case, Figure~\ref{fig:deterministic_and_stoch} (bottom) shows that, asymptotically, TD performs similar to MC irrespective of $k$, as the bootstrapping states are deterministic.
On the other hand, the variance of MC-A is no longer zero since the variance of the rewards cannot be explained by the actions.
This causes the $\mathrm{MSE}_\pi$ of DAE to converge slightly above Bellman iterations asymptotically.

\paragraph{The Low Coverage Case \normalfont{($|\actionspace|\in\{4, 16, 64\}$, $n=8$, $k=16$, $p_s=0$, $p_r=0$)}}
Finally, we consider the low coverage case, where the least squares objective become underdetermined for DAE.
In the context of RL, this means that most of the actions are not sampled, and the advantage estimates become unreliable.
We follow the common practice of choosing the minimum Euclidean-norm solution.
Figure~\ref{fig:coverage} shows how the small-sample bias affects learning in early phases (phase $<10$) and its relationship to the size of the action space.
Interestingly, we find that even when $|\actionspace|\geq 16$, where DAE converges to suboptimal solutions compared to the full coverage cases, DAE still converges to a lower $\mathrm{MSE}_\pi$ than TD, suggesting that the advantage estimates may help reduce variance even when they are poorly estimated.
\footnote{We find DAE converging to the true value function again using iterative solvers (see Appendix~\ref{app:additional_exp}).}

%We note that the initial decrease in MSE for $|\actionspace|\in\{16, 64\}$ is an artifact of the minimum Euclidean-norm solutions, which biases value estimates toward zero (the true value function in this case).

% \begin{figure}
%     \centering
%     \includegraphics[width=0.98\linewidth]{plots/exp4.pdf}
%     \caption{The stochastic reward case ($|\actionspace|{=}2$, $n{=}8$, $k\in\{4,16,64\}$, $p_s{=}0$, $p_r{=}0.2$). Lines and shadings represent (mean $\pm$ 3 standard error). Left/Right: Linear/Log $Y$-axis.}
%     \label{fig:stochastic_reward}
% \end{figure}
% \paragraph{The Stochastic Reward Case}
% In this case ($p_r\neq 0$), the DAE regression is no longer a perfect linear fit, and we can see the finite sample bias from Figure~\ref{fig:stochastic_reward} (the MSE of DAE converging slightly above the MC-A baseline.

\begin{figure*}[!h]
    \centering
    \includegraphics[width=0.9\linewidth]{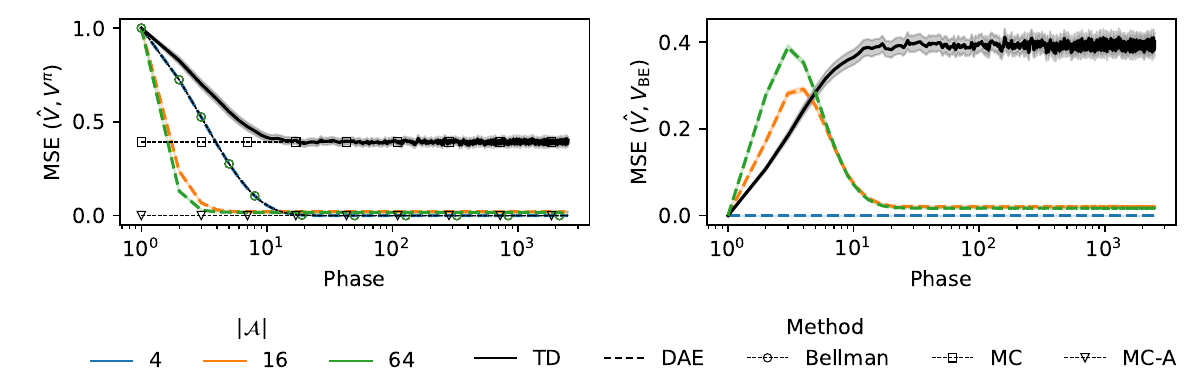}
    \caption{The low coverage case. Lines and shadings represent (mean $\pm$ 3 standard error). Note that TD (black, solid) does not depend on $|\actionspace|$.}
    \label{fig:coverage}
\end{figure*}
\section{Related Work}

The bias-variance tradeoff between MC and TD has been discussed in classical texts and surveys~\citep{sutton1998introduction, szepesvari2010algorithms, dann2014policy}; however, a rigorous analysis remained desirable.
\citet{kearns2000bias} analyzed the error bounds of multi-step TD learning in the phased setting~\citep{kearns1998finite}, which also serves as the basis of the present work.
In the batch setting,~\citet{grunewalder2007optimality} showed that LSTD~\citep{bradtke1996linear} is statistically more efficient than MC when the Markovian structure of the environment can provide additional information.
More recently, \citet{cheikhi2023statistical} derived a more precise statistical relationship between batch TD and MC by analyzing trajectory pooling.
However, the batch setting abstracts away the iterative nature of TD, and their implications in the online setting remain unclear.
Similar problems have also been studied in function approximation settings~\citep{tsitsiklis1996analysis, dalal2018finite, bhandari2018finite}, where convergence and finite-sample (time) error bounds were explored.

The advantage function~\citep{baird1995residual} is commonly used as control variates for policy gradient methods~\citep{sutton1999policy, greensmith2004variance}.
The present work demonstrates that the advantage function can also be used as control variates for policy evaluation, and shows that DAE~\citep{pan2022direct} can be seen as regression-adjusted control variates.

\section{Discussion}\label{sec:discussion}
We analyzed the behaviors of MC, TD, and DAE, and revealed one mechanism behind the variance reduction property of bootstrapping, namely, the ability to aggregate over a larger number of independent trajectories.
Furthermore, we established a connection between DAE and control variate regression, demonstrating how it can further reduce the variance of TD learning.
At its core, DAE exploits our knowledge of the policy to reduce variances and an interesting future direction would be to explore other types of control variates for policy evaluation. 

Finally, we note some limitations:
%(1) While previous work have shown that it is possible to extend results in synchronous (phased) settings to asynchronous settings by considering covering time~\citep{even2003learning}, it remains unclear how DAE would work in asynchronous settings, and to what extent the results hold under more realistic sampling (e.g., Markovian sampling) settings.
(1) In the present work, we considered only the phased and the IID settings, and it remains unclear to what extent the results generalize to more realistic sampling schemes (e.g., Markovian sampling).
(2) We focused mainly on the variance reduction property of DAE, but it should be noted that DAE also incurs additional space (store $\mmat$) and time (solve least-squares) complexities.
(3) The theoretical results for DAE only hold up to the large-sample approximation case where the control variate coefficients are optimal, and while empirical results seem to suggest that the variance reduction is beneficial even in the small-sample regime, a more rigorous analysis remains desirable.

\subsubsection*{Acknowledgments}
\label{sec:ack}
Hsiao-Ru Pan thanks Claire Vernade and Onno Eberhard for the fruitful discussions.
The authors thank the International Max Planck Research School for Intelligent Systems (IMPRS-IS) for supporting Hsiao-Ru Pan.
%%%%%%%%%%%%%%%%%%%%%%%%%%%%%%%%%%%%%%%%%%%%%%%%%%%%%%%%%%%%%%%%
%% NOTE: THIS MARKS THE END OF THE "MAIN TEXT"
%%%%%%%%%%%%%%%%%%%%%%%%%%%%%%%%%%%%%%%%%%%%%%%%%%%%%%%%%%%%%%%%

%%%%%%%%%%%%%%%%%%%%%%%%%%%%%%%%%%%%%%%%%%%%%%%%%%%%%%%%%%%%%%%%
%% Bibliography
%%%%%%%%%%%%%%%%%%%%%%%%%%%%%%%%%%%%%%%%%%%%%%%%%%%%%%%%%%%%%%%%
\bibliography{ref}

\begin{thebibliography}{30}
\providecommand{\natexlab}[1]{#1}
\providecommand{\url}[1]{\texttt{#1}}
\expandafter\ifx\csname urlstyle\endcsname\relax
  \providecommand{\doi}[1]{DOI: #1}\else
  \providecommand{\doi}{DOI: \begingroup \urlstyle{rm}\Url}\fi

\bibitem[Anderson et~al.(1999)Anderson, Bai, Bischof, Blackford, Demmel, Dongarra, Du~Croz, Greenbaum, Hammarling, McKenney, and Sorensen]{laug}
E.~Anderson, Z.~Bai, C.~Bischof, S.~Blackford, J.~Demmel, J.~Dongarra, J.~Du~Croz, A.~Greenbaum, S.~Hammarling, A.~McKenney, and D.~Sorensen.
\newblock \emph{{LAPACK} Users' Guide}.
\newblock Society for Industrial and Applied Mathematics, Philadelphia, PA, third edition, 1999.
\newblock ISBN 0-89871-447-8.

\bibitem[Asmussen \& Glynn(2007)Asmussen and Glynn]{asmussen2007stochastic}
S{\o}ren Asmussen and Peter~W Glynn.
\newblock \emph{Stochastic simulation: algorithms and analysis}, volume~57.
\newblock Springer, 2007.

\bibitem[Baird(1995)]{baird1995residual}
Leemon Baird.
\newblock Residual algorithms: Reinforcement learning with function approximation.
\newblock In \emph{Machine Learning Proceedings 1995}, pp.\  30--37. Elsevier, 1995.

\bibitem[Bhandari et~al.(2018)Bhandari, Russo, and Singal]{bhandari2018finite}
Jalaj Bhandari, Daniel Russo, and Raghav Singal.
\newblock A finite time analysis of temporal difference learning with linear function approximation.
\newblock In \emph{Conference on learning theory}, pp.\  1691--1692. PMLR, 2018.

\bibitem[Bradbury et~al.(2018)Bradbury, Frostig, Hawkins, Johnson, Leary, Maclaurin, Necula, Paszke, Vander{P}las, Wanderman-{M}ilne, and Zhang]{jax2018github}
James Bradbury, Roy Frostig, Peter Hawkins, Matthew~James Johnson, Chris Leary, Dougal Maclaurin, George Necula, Adam Paszke, Jake Vander{P}las, Skye Wanderman-{M}ilne, and Qiao Zhang.
\newblock {JAX}: composable transformations of {P}ython+{N}um{P}y programs, 2018.
\newblock URL \url{http://github.com/jax-ml/jax}.

\bibitem[Bradtke \& Barto(1996)Bradtke and Barto]{bradtke1996linear}
Steven~J Bradtke and Andrew~G Barto.
\newblock Linear least-squares algorithms for temporal difference learning.
\newblock \emph{Machine learning}, 22\penalty0 (1):\penalty0 33--57, 1996.

\bibitem[Burton \& R{\"o}sler(1995)Burton and R{\"o}sler]{burton1995l2}
Robert~M Burton and Uwe R{\"o}sler.
\newblock An l2 convergence theorem for random affine mappings.
\newblock \emph{Journal of applied probability}, 32\penalty0 (1):\penalty0 183--192, 1995.

\bibitem[Cheikhi \& Russo(2023)Cheikhi and Russo]{cheikhi2023statistical}
David Cheikhi and Daniel Russo.
\newblock On the statistical benefits of temporal difference learning.
\newblock In \emph{International Conference on Machine Learning}, pp.\  4269--4293. PMLR, 2023.

\bibitem[Dalal et~al.(2018)Dalal, Sz{\"o}r{\'e}nyi, Thoppe, and Mannor]{dalal2018finite}
Gal Dalal, Bal{\'a}zs Sz{\"o}r{\'e}nyi, Gugan Thoppe, and Shie Mannor.
\newblock Finite sample analyses for td (0) with function approximation.
\newblock In \emph{Proceedings of the AAAI Conference on Artificial Intelligence}, volume~32, 2018.

\bibitem[Dann et~al.(2014)Dann, Neumann, and Peters]{dann2014policy}
Christoph Dann, Gerhard Neumann, and Jan Peters.
\newblock Policy evaluation with temporal differences: A survey and comparison.
\newblock \emph{The Journal of Machine Learning Research}, 15\penalty0 (1):\penalty0 809--883, 2014.

\bibitem[Davidson \& MacKinnon(1992)Davidson and MacKinnon]{davidson1992regression}
Russell Davidson and James~G MacKinnon.
\newblock Regression-based methods for using control variates in monte carlo experiments.
\newblock \emph{Journal of Econometrics}, 54\penalty0 (1-3):\penalty0 203--222, 1992.

\bibitem[Greensmith et~al.(2004)Greensmith, Bartlett, and Baxter]{greensmith2004variance}
Evan Greensmith, Peter~L Bartlett, and Jonathan Baxter.
\newblock Variance reduction techniques for gradient estimates in reinforcement learning.
\newblock \emph{Journal of Machine Learning Research}, 5\penalty0 (9), 2004.

\bibitem[Grunewalder et~al.(2007)Grunewalder, Hochreiter, and Obermayer]{grunewalder2007optimality}
Steffen Grunewalder, Sepp Hochreiter, and Klaus Obermayer.
\newblock Optimality of lstd and its relation to mc.
\newblock In \emph{2007 International Joint Conference on Neural Networks}, pp.\  338--343. IEEE, 2007.

\bibitem[Hammersley \& Handscomb(1964)Hammersley and Handscomb]{Hammersley1964}
J.~M. Hammersley and D.~C. Handscomb.
\newblock \emph{Monte Carlo Methods}.
\newblock Springer Netherlands, 1964.
\newblock ISBN 9789400958197.
\newblock \doi{10.1007/978-94-009-5819-7}.
\newblock URL \url{http://dx.doi.org/10.1007/978-94-009-5819-7}.

\bibitem[Harris et~al.(2020)Harris, Millman, van~der Walt, Gommers, Virtanen, Cournapeau, Wieser, Taylor, Berg, Smith, Kern, Picus, Hoyer, van Kerkwijk, Brett, Haldane, del R{\'{i}}o, Wiebe, Peterson, G{\'{e}}rard-Marchant, Sheppard, Reddy, Weckesser, Abbasi, Gohlke, and Oliphant]{harris2020array}
Charles~R. Harris, K.~Jarrod Millman, St{\'{e}}fan~J. van~der Walt, Ralf Gommers, Pauli Virtanen, David Cournapeau, Eric Wieser, Julian Taylor, Sebastian Berg, Nathaniel~J. Smith, Robert Kern, Matti Picus, Stephan Hoyer, Marten~H. van Kerkwijk, Matthew Brett, Allan Haldane, Jaime~Fern{\'{a}}ndez del R{\'{i}}o, Mark Wiebe, Pearu Peterson, Pierre G{\'{e}}rard-Marchant, Kevin Sheppard, Tyler Reddy, Warren Weckesser, Hameer Abbasi, Christoph Gohlke, and Travis~E. Oliphant.
\newblock Array programming with {NumPy}.
\newblock \emph{Nature}, 585\penalty0 (7825):\penalty0 357--362, September 2020.
\newblock \doi{10.1038/s41586-020-2649-2}.
\newblock URL \url{https://doi.org/10.1038/s41586-020-2649-2}.

\bibitem[Kearns \& Singh(1998)Kearns and Singh]{kearns1998finite}
Michael Kearns and Satinder Singh.
\newblock Finite-sample convergence rates for q-learning and indirect algorithms.
\newblock \emph{Advances in neural information processing systems}, 11, 1998.

\bibitem[Kearns \& Singh(2000)Kearns and Singh]{kearns2000bias}
Michael~J Kearns and Satinder Singh.
\newblock Bias-variance error bounds for temporal difference updates.
\newblock In \emph{COLT}, pp.\  142--147, 2000.

\bibitem[Lin(2013)]{lin2013agnostic}
Winston Lin.
\newblock Agnostic notes on regression adjustments to experimental data: Reexamining freedman's critique.
\newblock \emph{The Annals of Applied Statistics}, pp.\  295--318, 2013.

\bibitem[Owen(2013)]{mcbook}
Art~B. Owen.
\newblock \emph{Monte Carlo theory, methods and examples}.
\newblock \url{https://artowen.su.domains/mc/}, 2013.

\bibitem[Paige \& Saunders(1982)Paige and Saunders]{paige1982lsqr}
Christopher~C Paige and Michael~A Saunders.
\newblock Lsqr: An algorithm for sparse linear equations and sparse least squares.
\newblock \emph{ACM Transactions on Mathematical Software (TOMS)}, 8\penalty0 (1):\penalty0 43--71, 1982.

\bibitem[Pan \& Sch{\"o}lkopf(2024)Pan and Sch{\"o}lkopf]{pan2024skill}
Hsiao-Ru Pan and Bernhard Sch{\"o}lkopf.
\newblock Skill or luck? return decomposition via advantage functions.
\newblock \emph{arXiv preprint arXiv:2402.12874}, 2024.

\bibitem[Pan et~al.(2022)Pan, G{\"u}rtler, Neitz, and Sch{\"o}lkopf]{pan2022direct}
Hsiao-Ru Pan, Nico G{\"u}rtler, Alexander Neitz, and Bernhard Sch{\"o}lkopf.
\newblock Direct advantage estimation.
\newblock \emph{Advances in Neural Information Processing Systems}, 35:\penalty0 11869--11880, 2022.

\bibitem[Puterman(2014)]{puterman2014markov}
Martin~L Puterman.
\newblock \emph{Markov decision processes: discrete stochastic dynamic programming}.
\newblock John Wiley \& Sons, 2014.

\bibitem[Sobel(1982)]{sobel1982variance}
Matthew~J Sobel.
\newblock The variance of discounted markov decision processes.
\newblock \emph{Journal of Applied Probability}, 19\penalty0 (4):\penalty0 794--802, 1982.

\bibitem[Sutton(1988)]{sutton1988learning}
Richard~S Sutton.
\newblock Learning to predict by the methods of temporal differences.
\newblock \emph{Machine learning}, 3\penalty0 (1):\penalty0 9--44, 1988.

\bibitem[Sutton et~al.(1998)Sutton, Barto, et~al.]{sutton1998introduction}
Richard~S Sutton, Andrew~G Barto, et~al.
\newblock Introduction to reinforcement learning.
\newblock 1998.

\bibitem[Sutton et~al.(1999)Sutton, McAllester, Singh, and Mansour]{sutton1999policy}
Richard~S Sutton, David McAllester, Satinder Singh, and Yishay Mansour.
\newblock Policy gradient methods for reinforcement learning with function approximation.
\newblock \emph{Advances in neural information processing systems}, 12, 1999.

\bibitem[Szepesvari(2010)]{szepesvari2010algorithms}
Csaba Szepesvari.
\newblock \emph{Algorithms for Reinforcement Learning}.
\newblock Morgan \& Claypool Publishers, 2010.

\bibitem[Tsitsiklis \& Van~Roy(1996)Tsitsiklis and Van~Roy]{tsitsiklis1996analysis}
John Tsitsiklis and Benjamin Van~Roy.
\newblock Analysis of temporal-diffference learning with function approximation.
\newblock \emph{Advances in neural information processing systems}, 9, 1996.

\bibitem[Virtanen et~al.(2020)Virtanen, Gommers, Oliphant, Haberland, Reddy, Cournapeau, Burovski, Peterson, Weckesser, Bright, {van der Walt}, Brett, Wilson, Millman, Mayorov, Nelson, Jones, Kern, Larson, Carey, Polat, Feng, Moore, {VanderPlas}, Laxalde, Perktold, Cimrman, Henriksen, Quintero, Harris, Archibald, Ribeiro, Pedregosa, {van Mulbregt}, and {SciPy 1.0 Contributors}]{2020SciPy-NMeth}
Pauli Virtanen, Ralf Gommers, Travis~E. Oliphant, Matt Haberland, Tyler Reddy, David Cournapeau, Evgeni Burovski, Pearu Peterson, Warren Weckesser, Jonathan Bright, St{\'e}fan~J. {van der Walt}, Matthew Brett, Joshua Wilson, K.~Jarrod Millman, Nikolay Mayorov, Andrew R.~J. Nelson, Eric Jones, Robert Kern, Eric Larson, C~J Carey, {\.I}lhan Polat, Yu~Feng, Eric~W. Moore, Jake {VanderPlas}, Denis Laxalde, Josef Perktold, Robert Cimrman, Ian Henriksen, E.~A. Quintero, Charles~R. Harris, Anne~M. Archibald, Ant{\^o}nio~H. Ribeiro, Fabian Pedregosa, Paul {van Mulbregt}, and {SciPy 1.0 Contributors}.
\newblock {{SciPy} 1.0: Fundamental Algorithms for Scientific Computing in Python}.
\newblock \emph{Nature Methods}, 17:\penalty0 261--272, 2020.
\newblock \doi{10.1038/s41592-019-0686-2}.

\end{thebibliography}
\bibliographystyle{rlj}

%%%%%%%%%%%%%%%%%%%%%%%%%%%%%%%%%%%%%%%%%%%%%%%%%%%%%%%%%%%%%%%%
% AUTHOR: If your paper has no supplementary materials, you may 
%         comment out the line below, which creates the title for
%         the supplementary materials.
%%%%%%%%%%%%%%%%%%%%%%%%%%%%%%%%%%%%%%%%%%%%%%%%%%%%%%%%%%%%%%%%
\beginSupplementaryMaterials

\section{Proofs}

\subsection{Proof of Lemma~\ref{lemma:var_mc}}
\label{app:proof_lemma_1}
\noindent\textbf{Lemma~\ref{lemma:var_mc}}
\begin{equation*}
    \Var(V_\mathrm{MC}(s))= \Var(\mathbf{G}) =
    \sum_{m=0}^\infty \gamma^{2km}\E\left[\left.\Var(\rmat + \gamma^k \pmat \vvec^\pi|s_0{=}s_{km})\right|s_0{=}s\right].
\end{equation*}
\begin{proof}
    Note that, when conditioned on $\gD$, both $\rmat$ and $\pmat$ become constants.
    Consequently, we have:
    \begin{align*}
        &\Var(V_\mathrm{MC}(s)) = \Var\left(\left. \mathbf{G} \right|s_0{=}s\right)\\
        &= 
        \Var\left(\left.\E\left[\left.\mathbf{G}\right| \gD, s_0{=}s\right]\right|s_0{=}s\right)+ \E\left[\left.\Var\left(\left.\mathbf{G}\right| \gD, s_0{=}s\right)\right|s_0{=}s\right] 
        \\
        &= \Var\left(\left. \rmat + \gamma^k \pmat \vvec^\pi\right|s_0{=}s\right) + \E\left[\left.\Var\left(\left.\frac{1}{n}\sum_{i=1}^n\sum_{t=k}^\infty\gamma^t r_t^i\right|\gD\right)\right|s_0{=}s\right] \\
        &= \Var\left(\left. \rmat + \gamma^k \pmat \vvec^\pi\right|s_0{=}s\right) + \frac{1}{n}\E\left[\left.\sum_{i=1}^n\Var\left(\gamma^kV_\mathrm{MC}(s_k^i)\right)\right|s_0{=}s\right] \\
        &= \Var\left(\left. \rmat + \gamma^k \pmat \vvec^\pi\right|s_0{=}s\right) + \gamma^{2k}\E\left[\Var\left(V_\mathrm{MC}(s_k)\right)|s_0{=}s\right].
    \end{align*}    
    Since $\Var(V_\mathrm{MC}(\cdot))$ is bounded, Lemma~\ref{lemma:var_mc} follows from expanding this recursion.
\end{proof}

\subsection{Proof of Theorem~\ref{thm:td_mc}} \label{app:proof_thm1}
\noindent\textbf{Theorem~\ref{thm:td_mc}}
\begin{equation*}
    \lim_{T\rightarrow\infty}\Var(V^T_{\mathrm{TD}(k)}(s))\leq \Var(V_\mathrm{MC}(s)).
\end{equation*}

\begin{proof}
    First, we show that $\lim_{T\rightarrow\infty}\Var(V^T_{\mathrm{TD}(k)}(s))$ converges.
    Note that
    \begin{align*}
      V^T_{\mathrm{TD}(k)}(s) = \rmat + \gamma^k \pmat\vvec^{T-1}_{\mathrm{TD}(k)},
    \end{align*}
    is a special case of random affine iterated system of the form:
    \begin{align}
        X_t = M_t X_{t-1} + N_t,
    \end{align}
    where $(M_t, N_t)$ are IID random variables.
    Furthermore, since $\gamma^k\pmat$ is a contraction respect to $||\cdot||_\infty$ and $\rmat$ is bounded, we know that the random vector $\vvec^T_{\mathrm{TD}(k)}$ converges in distribution with respect to the Wassterstein distance $W_\infty$~\citep{burton1995l2}.
    Consequently, all finite moments of $\vvec^T_{\mathrm{TD}(k)}$ also converges as $T\rightarrow\infty$.
    
    By Lemma~\ref{lemma:var_td}, we have:
    \begin{align*}
    \Var(V^T_{\mathrm{TD}(k)}(s)) \leq \sum_{m=0}^{T-1}\gamma^{2km} \E\left[\left.\Var\left(\left.\rmat + \gamma^k\pmat\bar{\vvec}^{T-1-m}_{\mathrm{TD}(k)}\right|s_0{=}s_{km}\right)\right|s_0{=}s\right].
    \end{align*}
    It is enough to show that the summation on the right hand side converges to $\Var(V_\mathrm{MC}(s))$ as $T\rightarrow\infty$.
    Let $x_{T-m,m} = \E\left[\left.\Var\left(\left.\rmat + \gamma^k\pmat\bar{\vvec}^{T-1-m}_{\mathrm{TD}(k)}\right|s_0{=}s_{km}\right)\right|s_0{=}s\right]$, we are interested in the following limit
    \begin{equation*}
    \lim_{T\rightarrow\infty}\sum_{m=0}^{T-1}\gamma^{2km}x_{T-m,m}.
    \end{equation*}
    Since $x_{\infty, m}\coloneq\lim_{T\rightarrow\infty} x_{T-m,m} = \E\left[\left.\Var\left(\left.\rmat + \gamma^k\pmat \vvec^\pi\right|s_0{=}s_{km}\right)\right|s_0{=}s\right]$, there exists $N\in\sN$ such that if $T-m > N$ then $|x_{T-m,m} - x_{\infty, m}| < \epsilon$.
    In addition, since $\rmat$, $\pmat$ and $\vvec^\pi$ are all bounded, there exists $M \in\sR$ such that $|x_{\infty,m}| < M$ and $|x_{T-m, m}|<M$ for all $m, T-m\in\sZ_{+}$.
    Consequently,
    \begin{align*}
        &\left|\sum_{m=0}^{T-1}\gamma^{2km}x_{T-m,m} - \sum_{m=0}^\infty \gamma^{2km}x_{\infty,m}\right| \\
        \leq& \left|\sum_{m=0}^{T-1}\gamma^{2km}(x_{T-m,m} -x_{\infty,m})\right| + \left|\sum_{m=T}^\infty \gamma^{2km}x_{\infty,m}\right| \\
        \leq& \left|\sum_{m=0}^{T-n-1}\gamma^{2km}(x_{T-m,m} -x_{\infty,m})\right| + 
        \left|\sum_{m=T-n}^{T-1}\gamma^{2km}(x_{T-m,m} -x_{\infty,m})\right| + 
        \frac{M\gamma^{2kT}}{1-\gamma^{2k}}\\
        \leq& \frac{\epsilon}{1-\gamma^{2k}} + 
        \frac{2M\gamma^{2k(T-n)}}{1-\gamma^{2k}} + 
        \frac{M \gamma^{2kT}}{1-\gamma^{2k}}
    \end{align*}
    is arbitrarily small as $T\rightarrow\infty$, and
    \begin{equation*}
    \lim_{T\rightarrow\infty}\Var(V^T_{\mathrm{TD}(k)}(s)) \leq\lim_{T\rightarrow\infty}\sum_{m=0}^{T-1}\gamma^{2km}x_{T-m,m}=\sum_{m=0}^{\infty}\gamma^{2km}x_{\infty,m} = \Var(V_\mathrm{MC}(s)).
    \end{equation*}
        
\end{proof}

\paragraph{Extension to IID Sampling}

In the IID setting, we sample states independently from a fixed distribution
$\mu$ (e.g., the occupancy measure) and update their values accordingly.
Assuming that $\mu$ has full support over the state space $\statespace$, the above analysis can be extended to the asynchronous setting with minor modifications.
Specifically, at each iteration $T$, we sample a starting state $s\sim\mu$ and update its value using the $k$-step TD learning rule.
Accordingly, $\gD$ now denotes a collection of $k$-step trajectories originating from $s$, rather than samples covering the entire state space $\statespace$.
We are now ready to proof the following theorem:
\begin{theorem}
    \label{thm:td_mc_iid}
    Given distribution $\mu$ over $\statespace$. We define the IID phased TD iteration as a two-step process:
    (1) sample a state $s\sim\mu$, and (2) update its value estimate by
    \begin{align*}
        V^T_{\mathrm{TD}(k)}(\tilde{s}) =
        \begin{cases}
            \rmat + \gamma^k \pmat\vvec^{T-1}_{\mathrm{TD}(k)} &  \tilde{s} = s \\
            V^{T-1}_\mathrm{TD}(\tilde{s}) & \text{otherwise}
        \end{cases}.
    \end{align*}
    If $\mu$ covers $\statespace$, then $\lim_{T\rightarrow\infty}\Var(V^T_{\mathrm{TD}(k)}(s))\leq \Var(V_\mathrm{MC}(s))$.
\end{theorem}
\begin{proof}
    Under this sampling scheme, the expected value vector evolves according to a slightly different equation $\bar{V}^T_{\mathrm{TD}(k)}(s) = \mu(s)(\E[\rmat] + \gamma^k\E[\pmat]\bar\vvec^{T-1}_{\mathrm{TD}(k)}) + (1 - \mu(s))\bar{V}^{T-1}_{\mathrm{TD}(k)}(s)$.
    If $\mu$ covers $\statespace$ (i.e., $\mu(s) > 0$ for all $s\in\statespace$), then $\bar{V}^T_{\mathrm{TD}(k)}(s)$ converges to $\vvec^\pi$.
    Note that Lemma~\ref{lemma:var_td} does not depend on the sampling distribution $\mu$, and continues to hold for the sampled state.
    However, we now need to account for the variance from whether a state gets updated or not.
    More precisely, we have
    \begin{align*}
        \Var(V^{T+1}_{\mathrm{TD}(k)}(s)) = 
        &\mu(s)\Var(\rmat + \gamma^k \pmat\vvec^{T}_{\mathrm{TD}(k)}) + (1-\mu(s))\Var(V^{T}_{\mathrm{TD}(k)}(s)) \\
        &+ \mu(s)(1-\mu(s))\left(\E[\rmat + \gamma^k \pmat\vvec^{T}_{\mathrm{TD}(k)}]-\E[V^{T}_{\mathrm{TD}(k)}(s)]\right)^2.  
    \end{align*}
    We now denote
    \begin{equation*}
        \Delta^T = \left(\E[\rmat + \gamma^k \pmat\vvec^{T}_{\mathrm{TD}(k)}]-\E[V^{T}_{\mathrm{TD}(k)}(s)]\right)^2.
    \end{equation*}
    Since the expectation of $\vvec^{T}_{\mathrm{TD}(k)}$ converges to $\vvec^\pi$, we have $\Delta^T\rightarrow 0$.
    Expanding this recursion gives
    \begin{align*}
        \Var(V^{T+1}_{\mathrm{TD}(k)}(s)) =& \\
        &\mu(s)\left(
        \sum_{t=0}^T(1-\mu(s))^t\Var(\rmat + \gamma^k \pmat\vvec^{T-t}_{\mathrm{TD}(k)}) 
        + \sum_{t=0}^T(1-\mu(s))^{t+1} \Delta^{T-t}\right).  
    \end{align*}
    In the limit $T\rightarrow \infty$, the first sum converges to $\lim_{T\rightarrow\infty}\Var(\rmat + \gamma^k \pmat\vvec^{T}_{\mathrm{TD}(k)})$, and the second sum converges to $\lim_{T\rightarrow\infty}\Delta^{T-t}=0$.
    The results then follow from Lemma~\ref{lemma:var_td}.
\end{proof}
Remark: The key requirement for this result is that trajectory sampling be independent of the value estimates.
This requirement may be violated in the Markovian sampling setting, where the value estimates can be correlated with the next-state distribution (i.e.,
$\pmat$ and $\vvec^{T-1}$ are dependent).
Such dependence invalidates the conditional variance bound $\Var\left(\left.\pmat\vvec^{T-1}_{\mathrm{TD}(k)}\right|\gD\right) \leq \pmat\Var\left(\vvec^{T-1}_{\mathrm{TD}(k)}\right)$, which is crucial for the analysis.

\subsection{Proof of Corollary~\ref{cor:sample_efficiency}} \label{app:proof_cor1}
\noindent\textbf{Corollary~\ref{cor:sample_efficiency}}
If $\vvec^{\tau}_{\mathrm{TD}(k)}{=}\vvec^{\tau}_{\mathrm{TD}(kT)}{=}\vvec$, then
\begin{equation*}
    \Var(V^{\tau+T}_{\mathrm{TD}(k)}(s)|\vvec^{\tau}_{\mathrm{TD}(k)}{=}\vvec) \leq 
    \Var(V^{\tau+1}_{\mathrm{TD}(kT)}(s)|\vvec^{\tau}_{\mathrm{TD}(kT)}{=}\vvec).
\end{equation*}

\begin{proof}
    Without loss of generality, we let $\tau=0$, since $\vvec^0$ is arbitrary.
    We proof by induction on $T$.
    For $T=1$, equality holds.
    For $T>1$, assume the inequality holds for $T-1$.
    By Lemma~\ref{lemma:var_td}, we have:
    \begin{align*}
        \Var(V^T_{\mathrm{TD}(k)}(s)) \leq \Var\left(\left. \rmat + \gamma^k\pmat\bar{\vvec}^{T-1}_{\mathrm{TD}(k)}\right|s_0{=}s\right)
    + \gamma^{2k}\E_{s_k}\left[\left.\Var\left(V^{T-1}_{\mathrm{TD}(k)}(s_k)\right)\right|s_0{=}s\right].
    \end{align*}
    Similar to Lemma~\ref{lemma:var_mc}, we can break down the variance of long trajectories into segments, which leads to:
    \begin{align*}
        \Var(V^1_{\mathrm{TD}(kT)}(s)) &= \Var\left(\left. \frac{1}{n}\sum_{i=1}^n\left(\sum_{t=0}^{kT-1} \gamma^t r_t^i + \gamma^{kT}V^0_{\mathrm{TD}(kT)}(s_{kT}^i)\right)\right|s_0{=}s\right) \\
        &= \Var\left(\left. \rmat + \gamma^k\pmat\bar{\vvec}^{1}_{\mathrm{TD}(k(T-1))}\right|s_0{=}s\right) + \gamma^{2k} \E_{s_k}\left[\Var(V^1_{\mathrm{TD}(k(T-1))}(s_k))|s_0{=}s\right]
    \end{align*}
    Recall that the expected value estimates follow the Bellman update, so $\bar{\vvec}^{T-1}_{\mathrm{TD}(k)} = \bar{\vvec}^{1}_{\mathrm{TD}(k(T-1))}$.
    Since the inequality holds for $T-1$, we have:
    \begin{align*}
        \E_{s_k}\left[\left.\Var\left(V^{T-1}_{\mathrm{TD}(k)}(s_k)\right)\right|s_0{=}s\right]
        \leq \E_{s_k}\left[\left.\Var\left(V^{1}_{\mathrm{TD}(k(T-1))}(s_k)\right)\right|s_0{=}s\right],
    \end{align*}
    which concludes the proof.
\end{proof}

\subsection{Proof of Theorem~\ref{thm:dae}}\label{app:proof_thm2}
\noindent\textbf{Theorem~\ref{thm:dae}}
\begin{equation*}
 \limsup_{T\rightarrow\infty} \Var(V^{T}_{\mathrm{DAE}^*(k)}(s)) \leq \Var(V_{\sufmca}(s)).
\end{equation*}
\begin{proof}
    Note that $V^T_{\mathrm{DAE}^*(k)}(s)$ is DAE with the optimal control variate coefficients $A^{*T}$.
    Consequently, we must have:
    \begin{align}\label{eqn:dae_api}
       \Var(V^{T}_{\mathrm{DAE}^*(k)}(s))&= \Var\left(\left.\rmat + \gamma^k \pmat\vvec^{T-1}_{\mathrm{DAE}(k)} - \mmat\avec^{*T}\right|s_0{=}s\right)\\
       &\leq \Var\left(\left.\rmat + \gamma^k \pmat\vvec^{T-1}_{\mathrm{DAE}(k)} - \mmat\avec^\pi\right|s_0{=}s\right),
    \end{align}
    where $\avec^\pi$ is the true advantage function.
    Since $\avec^\pi$ is a constant vector, and $\mmat$ is uniformly bounded, we can rewrite the right hand side of this inequality by
    \begin{equation}
        \Var\left(\left.\rmat' + \gamma^k \pmat\vvec^{T-1}_{\mathrm{DAE}(k)}\right|s_0=s\right),
    \end{equation}
    where $\rmat' = \rmat - \mmat\avec^\pi$.
    By Corollay~\ref{thm:td_mc}, we know that this variance is bounded above by the variance of the MC estimator with this new reward function, which is precisely the variance of $V_{\sufmca}$.
\end{proof}
Finally, we make a remark about this corollary.
Since the advantage estimate $\avec^T$ now also depends on the bootstrapping value function, it is not clear whether the update remains a contraction, or whether higher moments (e.g., variance) also converge.
As such, we only prove the supremum limit is upper bounded by $V_{\sufmca}$.

\section{Experimental Details}
\label{app:exp_details}
Algorithm~\ref{alg:code} shows the pseudocode.
All experiments are based on Python with least-square solvers implemented by \texttt{NumPy}~\citep{harris2020array}, \texttt{SciPy}~\citep{2020SciPy-NMeth} or \texttt{JAX}~\citep{jax2018github}.
A single run (1 seed, 2500 phases) takes less than a minute on commercial CPUs, except for the large $|\actionspace|$ experiment, where we leveraged GPUs (Nvidia A100) to parallelize the least-square solver.
We use LSQR~\citep{paige1982lsqr} as the default least-square solver as we found it to be slightly faster.
The only exception is the large $|\actionspace|$ experiment, where we used the SVD-based minimum norm solver~\citep{laug} to ensure reproducibility.

\begin{algorithm}
\caption{Phased TD/DAE}
\label{alg:code}
\begin{algorithmic}[1] %[1] enables line numbers
\REQUIRE $n$, $k$, \texttt{alg}$\in$\{TD, DAE\}, \texttt{LSTSQ\_SOLVER}
\STATE Initialize $\vvec\equiv 0$
\FOR{$T=1, 2, \dots$}
    \STATE $\gD=\{\}$    
    \FOR{$s\in\statespace$}
        \FOR{$i=1,\dots, n$}
            \STATE Sample $k$-step trajectory $\tau$ from environment
            \STATE $\gD\leftarrow \gD\cup\{\tau\}$
        \ENDFOR
    \ENDFOR
    \IF{\texttt{alg} == TD}
        \STATE Compute $\rmat$, $\pmat$ from $\gD$
        \STATE $\vvec\leftarrow\rmat + \gamma^k\pmat \vvec$
    \ELSE
        \STATE Compute $\mmat$, $\rmat$, $\pmat$ from $\gD$
        \STATE $\vvec, \avec\leftarrow\mathtt{LSTSQ\_SOLVER}(||\rmat + \gamma^k\pmat \vvec - \hat{\vvec} - \mmat \hat{\avec}||^2)$
    \ENDIF
\ENDFOR
\end{algorithmic}
\end{algorithm}

\subsection{Additional Experiments}
\label{app:additional_exp}

\begin{figure}
    \centering
    \includegraphics[width=0.98\linewidth]{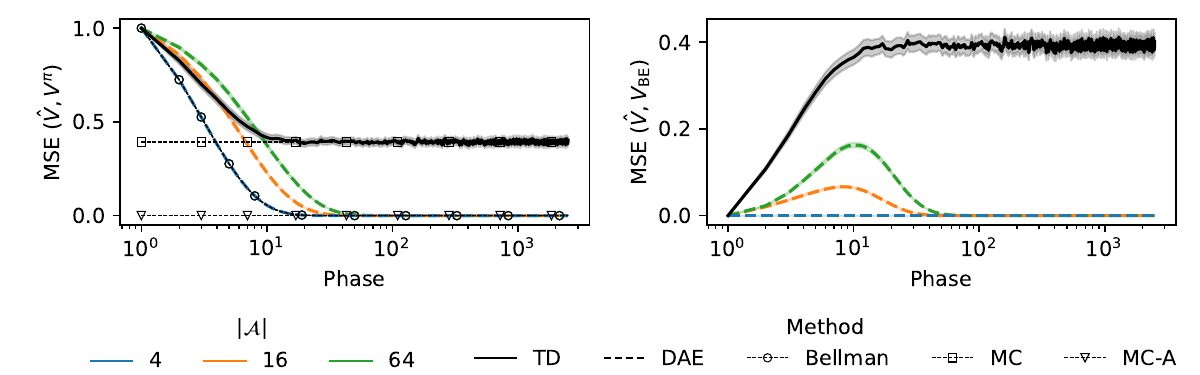}
    \caption{The low coverage case with an iterative solver. Lines and shadings represent (mean $\pm$ 3 standard error). Note that TD does not depend on $|\actionspace|$.}
    \label{fig:coverage_iter}
\end{figure}
\paragraph{The Low Coverage Case With an Iterative Solver \normalfont{($|\actionspace|\in\{4, 16, 64\}$, $n=8$, $k=16$, $p_s=0$, $p_r=0$)}}
In Section~\ref{sec:exp}, we showed that increasing the size of the action space results in DAE converging to suboptimal solutions when regularized with minimum norm solutions.
In Figure~\ref{fig:coverage_iter}, we rerun the same experiment but with an iterative least-sqaures solver (LSQR~\citep{paige1982lsqr} in this case), where the optimum in the previous phase is used as the initialization for the current phase.
We find that DAE converges again to the true value function, although at a slower rate as $|\actionspace|$ increases.
This might partially explain the success of DAE in the deep RL setting~\citep{pan2022direct}, where gradient-based optimization is used.

\end{document}